\documentclass[twoside]{article}

\usepackage[accepted]{aistats2021}
\usepackage[round]{natbib}

\usepackage{titletoc}

\usepackage[dvipsnames]{xcolor}
\usepackage[colorlinks=true, linkcolor=OliveGreen, citecolor=OliveGreen, pagebackref]{hyperref}
\usepackage{microtype}

\usepackage{algpseudocode}
\usepackage[ruled,vlined]{algorithm2e}

\usepackage{mathtools}
\usepackage{amsmath}
\usepackage{amsthm}
\usepackage{bbm}
\usepackage{bm}
\usepackage{amsfonts}
\usepackage{amssymb}

\makeatletter
\newcommand\etc{etc\@ifnextchar.{}{.\@}}
\makeatother

\makeatletter
\newcommand\ie{i.e\@ifnextchar.{}{.\@}}
\makeatother

\makeatletter
\newcommand\iid{i.i.d\@ifnextchar.{}{.\@}}
\makeatother

\makeatletter
\newcommand\wrt{w.r.t\@ifnextchar.{}{.\@}}
\makeatother

\theoremstyle{definition}  

\theoremstyle{plain}

\newtheorem{theorem}{Theorem}
\newtheorem{lemma}{Lemma}

\theoremstyle{remark}
\newtheorem*{remark}{Remark}

\usepackage{prettyref}
\newcommand{\pref}[1]{\prettyref{#1}}

\newcommand{\savehyperref}[2]{\texorpdfstring{\hyperref[#1]{#2}}{#2}}
\newrefformat{eq}{\savehyperref{#1}{\textup{(\ref*{#1})}}}
\newrefformat{eqn}{\savehyperref{#1}{Equation~\ref*{#1}}}
\newrefformat{lem}{\savehyperref{#1}{Lemma~\ref*{#1}}}
\newrefformat{fact}{\savehyperref{#1}{Fact~\ref*{#1}}}
\newrefformat{def}{\savehyperref{#1}{Definition~\ref*{#1}}}
\newrefformat{thm}{\savehyperref{#1}{Theorem~\ref*{#1}}}
\newrefformat{corr}{\savehyperref{#1}{Corollary~\ref*{#1}}}
\newrefformat{sec}{\savehyperref{#1}{Section~\ref*{#1}}}
\newrefformat{app}{\savehyperref{#1}{Appendix~\ref*{#1}}}
\newrefformat{ass}{\savehyperref{#1}{Assumption~\ref*{#1}}}
\newrefformat{ex}{\savehyperref{#1}{Example~\ref*{#1}}}
\newrefformat{fig}{\savehyperref{#1}{Figure~\ref*{#1}}}
\newrefformat{alg}{\savehyperref{#1}{Algorithm~\ref*{#1}}}
\newrefformat{rem}{\savehyperref{#1}{Remark~\ref*{#1}}}
\newrefformat{conj}{\savehyperref{#1}{Conjecture~\ref*{#1}}}
\newrefformat{prop}{\savehyperref{#1}{Proposition~\ref*{#1}}}
\newrefformat{proto}{\savehyperref{#1}{Protocol~\ref*{#1}}}
\newrefformat{prob}{\savehyperref{#1}{Problem~\ref*{#1}}}
\newrefformat{claim}{\savehyperref{#1}{Claim~\ref*{#1}}}

\DeclarePairedDelimiter{\abs}{\lvert}{\rvert} %
\DeclarePairedDelimiter{\brk}{[}{]}
\DeclarePairedDelimiter{\crl}{\{}{\}}
\DeclarePairedDelimiter{\prn}{(}{)}
\DeclarePairedDelimiter{\nrm}{\|}{\|}

\let\Pr\undefined
\DeclareMathOperator{\Pr}{\mathrm{Pr}}

\DeclareMathOperator*{\argmin}{arg\,min} 
\DeclareMathOperator*{\argmax}{arg\,max}

\newcommand{\ind}{\mathbb{I}}    

\def\ddefloop#1{\ifx\ddefloop#1\else\ddef{#1}\expandafter\ddefloop\fi}
\def\ddef#1{\expandafter\def\csname #1bb\endcsname{\ensuremath{\mathbb{#1}}}}
\ddefloop ABCDEFGHIJKLMNOPQRSTUVWXYZ\ddefloop
\def\ddefloop#1{\ifx\ddefloop#1\else\ddef{#1}\expandafter\ddefloop\fi}
\def\ddef#1{\expandafter\def\csname #1b\endcsname{\ensuremath{\mathbf{#1}}}}
\ddefloop ABCDEFGHIJKLMNOPQRSTUVWXYZ\ddefloop
\def\ddef#1{\expandafter\def\csname #1c\endcsname{\ensuremath{\mathcal{#1}}}}
\ddefloop ABCDEFGHIJKLMNOPQRSTUVWXYZ\ddefloop
\def\ddef#1{\expandafter\def\csname #1hat\endcsname{\ensuremath{\widehat{#1}}}}
\ddefloop ABCDEFGHIJKLMNOPQRSTUVWXYZ\ddefloop
\def\ddef#1{\expandafter\def\csname hc#1\endcsname{\ensuremath{\widehat{\mathcal{#1}}}}}
\ddefloop ABCDEFGHIJKLMNOPQRSTUVWXYZ\ddefloop
\def\ddef#1{\expandafter\def\csname #1til\endcsname{\ensuremath{\widetilde{#1}}}}
\ddefloop ABCDEFGHIJKLMNOPQRSTUVWXYZ\ddefloop
\def\ddef#1{\expandafter\def\csname tc#1\endcsname{\ensuremath{\widetilde{\mathcal{#1}}}}}
\ddefloop ABCDEFGHIJKLMNOPQRSTUVWXYZ\ddefloop

\renewcommand{\epsilon}{\varepsilon}
\DeclareMathOperator{\ex}{\mathbb{E}}
\newcommand{\inner}[1]{\langle #1\rangle}

\newcommand{\indic}[1]{\ind_{\crl{#1}}}

\newcommand{\mpick}{\textsc{Model Picker}}

\newcommand{\emo}{\textsc{EmoContext}}
\newcommand{\imagenet}{\textsc{ImageNet}}
\newcommand{\drift}{\textsc{Drift}}

\newcommand{\cifar}{\textsc{CIFAR-10}}

\newcommand{\entropy}{\textsc{Entropy}}
\newcommand{\importance}{\textsc{Importance}}
\newcommand{\efficient}{\textsc{Efficient}}
\newcommand{\efal}{\textsc{Efal}}
\newcommand{\sqbc}{\textsc{S-QBC}}
\newcommand{\qbc}{\textsc{QBC}}
\newcommand{\LEP}{\textsc{LEP}}

\newcommand{\cifarw}{\textsc{CIFAR-10 V2}}

\newcommand{\loss}{\bm{\ell}}
\newcommand{\ellest}{\hat{\bm{\ell}}}
\newcommand{\Lest}{\widehat{\bm{L}}}
\newcommand{\Loss}{\bm{L}}
\newcommand{\dist}{\bm{w}}
\newcommand{\sample}{x}
\newcommand{\pred}{\bm{p}}
\newcommand{\query}{q}
\newcommand{\queryind}{Q}
\newcommand{\querybef}{v}

\usepackage{enumitem}
\usepackage{microtype}
\usepackage[textsize=tiny, textwidth=5em]{todonotes}
\usepackage{subcaption}

\setuptodonotes{fancyline, color=blue!30,}

\begin{document}

%
\runningtitle{Online Active Model Selection}

%
\runningauthor{M. Karimi, N. G\"urel, B. Karla\v{s}, J. Rausch, C. Zhang, A. Krause}
\twocolumn[

\aistatstitle{Online Active Model Selection for Pre-trained Classifiers}

\aistatsauthor{ Mohammad Reza Karimi$^\star$ \And Nezihe Merve G\"urel$^\star$ \And Bojan Karla\v{s} }
\aistatsauthor{Johannes Rausch \And Ce Zhang \And Andreas Krause }

\aistatsaddress{ ETH Z\"urich } ]

\begin{abstract}
Given $k$ pre-trained classifiers and a stream of unlabeled data examples, how can we actively decide when to query a label so that we can distinguish the best model from the rest while making a small number of queries?
Answering this question has a profound impact on a range of practical scenarios.
In this work, we design an online selective sampling approach that actively selects informative examples to label and outputs the best model with high probability at any round. Our algorithm can be used for online prediction tasks for both adversarial and stochastic streams.
We establish several theoretical guarantees for our algorithm and extensively demonstrate its effectiveness in our experimental studies.
\end{abstract}


\section{INTRODUCTION}
Model selection from a set of pre-trained models is an emerging problem in
machine learning and has implications in several practical scenarios.
Industrial examples include cases in which a telecommunication company or a
flight booking company has multiple ML models trained over different sliding
windows of data and hopes to pick the one that performs the best on a given day.
For many real-world problems, unlabeled data is abundant and can be
inexpensively collected, while labels are expensive to acquire and require human
expertise.  Consequently, there is a need to robustly identify the best model
under limited labeling resources. Similarly, one often needs reasonable
predictions for the unlabeled data while keeping the labeling budget low.

Depending on the data availability, one can consider two settings: (i) the
\emph{pool-based} setting assumes that the learner has access to a pool of
unlabeled data, and she can select informative data samples from the pool to
achieve her task, and (ii) the \emph{online} setting assumes the data is
arriving one example at a time (\ie, in a stream), and the learner decides to
ask for the example's label on the go or to just throw it away.  While offering
fewer options on which data to label, this setting alleviates the scalability
challenge of storing and processing a large pool of examples in the pool-based
setting.

Another important aspect is the nature of the data: the instance-label pairs
might be sampled \emph{\iid}\ from a fixed distribution, or chosen
\emph{adversarially} by an adversary. While sometimes the
\iid~assumption is reasonable, there are practical scenarios where this
assumption fails to hold. These include cases where there are temporal or
spatial dependencies or non-stationarities in the dataset/stream. In these situations,
it may be safer not to make assumptions on the data and rather consider
worst-case data streams.
\paragraph{Contributions}
We develop a novel, principled and efficient model selection approach --\mpick-- for the online
setting.
Our query strategy is randomized
and leverages hypothetical query answers to decide which data
examples are likely to be informative for identifying the best model.  We prove
that our algorithm has no regret for adversarial streams, \ie, its performance
for sequential label prediction is close to the best model for that stream in
hindsight. Our bounds match (up to a constant) those of existing online
algorithms  that have access to all labels. We also establish bounds on the
number of label queries and the quality of the output model of \mpick. We
furthermore conduct extensive experiments, comparing our algorithm with a range of
other methods. To reach the same accuracy, competing methods can often require
up to 2.5$\times$ more labels.  Apart from the relative performance, on the
\textsc{ImageNet} dataset, \mpick\ requires a mere 13\% labeled instances to
select the best among 102 pre-trained models with 90\% confidence, while having
up to 1.3$\times$ lower regret. These results establish \mpick\ as the
state-of-the-art for this problem. We also make everything open and
reproducible.\footnote{The code is available at
\url{https://github.com/DS3Lab/online-active-model-selection}}


\section{RELATED WORKS}
Our approach relates to several bodies of literature. 
For each related area, we reference similar works that match the objective of
our paper.

\paragraph{Active Model Selection}
\citet{Madani2004ActiveModelSelection} develop their method for the online
setting. They seek to identify the best model via probing models, one at a time,
with \iid\ samples, while having a fixed budget for the number of probes. In
contrast, our approach applies even to adversarial streams and allows one to make
predictions online, while minimizing the number of queries made.  Most of other
previous works~\citep{Sawade2012ActiveModelComparison, Gardner2015BayesianAMS,
Alnur2014ALModelSelection, Sawade2010ActiveRiskEstimation, Katariya2012Active,
Kumar2018ClassRiskEst, Leite2010ActiveTesting} focus on pool-based sampling of
informative instances, where the learner ranks the entire pool of unlabeled data
and greedily selects the most informative examples. This setting substantially
differs from the streaming setting, and we focus on the latter for reasons of
scalability and applicability to many real-world situations.

\paragraph{Active Learning}
Active learning aims to query the label of those instances that help improving
the \emph{training} of classifiers, rather than selecting among pre-trained
models. Here we review those methods that can potentially be adapted for model
selection.  The celebrated query-by-committee (\qbc)
paradigm~\citep{Seung1992QBC} forms a committee of classifiers to vote on the
labeling of incoming examples.  The query decision is made based on the degree
of disagreement among the committee members. The general strategy is to query those
instances that help the learner prune the committee and only keep those classifiers
with higher accuracies. There are other \qbc\ approaches in active learning,
such as
\citet{Cohn1994ImpGenAL, Mccallum1998poolAL, Mamitsuka1998Boost, Melville2004Diverseensembles,
Settles2008Sequence, Zhu2007Active}. One limitation of these algorithms is that
they often focus on pool-based sampling, which limits
their scalability.  Several other approaches consider active learning in the
streaming setting. The seminal works of \citet{Dasgupta2008Agnostic} and \citet{Balcan2009AgnosticAL}, followed by
\citet{Beygelzimer2010Agnostic, Zhang2014DisagAL}, use
disagreement-based strategies.
The idea of using importance weights in active learning is studied by a series
of works including~\citet{Beygelzimer2008ImpWeightedAL, Masashi2006ALmisspec,
Beygelzimer2011EfficientAL, Bach2007ALMis}, where importance weights are
introduced to correct sampling bias and provide statistically consistent
convergence to the optimal classifier in the PAC learning setting. All of the
above approaches on stream-based active learning focus on \iid\ streams and try
to improve the supervised training of classifiers, whereas our approach applies
to the more general adversarial streams and performs no training.

\paragraph{Online Learning and Bandits}
Sequential label prediction is an important problem in online learning. The
setting closest to ours is {\em label-efficient prediction} (\LEP)
\citep{cesa2005minimizing}, where they query the label with a \emph{fixed}
probability at each round, and that probability also appears in the regret
bound. However, we use the side information of the models predictions to
\emph{adapt the probability} of querying to the information content of the
instance at hand, thereby significantly reducing the required labels in practice
and lowering the regret, as demonstrated theoretically and in our experiments.
Moreover, there is no study of the quality of the model outputted at the end of
the stream, for neither adversarial nor stochastic streams.  Another problem
similar to ours is {\em consistent online
learning}~\citep{jaghargh19a,altschuler2018online}, where the learner seeks to
minimize the number of switches of her actions, while observing the
loss every round, even if she does not update her strategy. In our setting,
however, we do {\em not} know the loss in the rounds we do not query.
Similar challenges arise in the multi-armed bandit literature.  In a way, our
setup lies between the usual prediction with experts advice and multi-armed
bandit problems. Our algorithm is related in spirit to the EXP3
algorithm~\citep{auer2002finite} for adversarial bandits. The key difference
is that EXP3 uses the probability of selecting an arm to construct an
unbiased loss estimator, whereas we consider the probability of observing the
\emph{whole} loss. While similar in spirit, the standard EXP3 analysis fails to
yield a regret bound, as discussed in the footnote of page \pageref{rem:on-omd}.

\section{PROBLEM STATEMENT AND BACKGROUND}

\looseness -1%
Assume that we have $k$ pre-trained classifiers (\emph{experts}). Let $\Xc$ and
$\Cc$ be the set of all possible inputs and classes, respectively. Our
sequential prediction problem is a game played in
rounds. Consider a stream of data $\crl{(\sample_t,c_t)\in\Xc\times \Cc}_{t\geq
1}$ generated by an unknown mechanism. At round $t$, $\sample_t$ together with
all classifiers predictions $\pred_t\in\Cc^k$ is revealed to the learner. She
then selects one of the experts $I_t\in[k]$\footnote{In here and what follows,
$[k]=\crl{1,\ldots,k}$.} and incurs a loss of 1 if that expert
misclassifies $\sample_t$. Finally, the learner decides whether to query the label $c_t$.
If no query is made, then $c_t$ remains hidden, otherwise, the learner observes
$c_t$ and the loss $\loss_t \in \{0,1\}^k$ defined as $\loss_{t,i} =
\indic{\pred_{t,i} \neq c_t}$, with $\indic{\cdot}$ being the indicator function.
Note that
$I_t$ can only depend on the past inputs and the observed labels. The goal of
the learner is to select $I_t$ in such a way that up to any round $T$, the total
misclassifications she makes is close to the total mistakes of the best expert
up to time $T$ in hindsight. This performance measure is formalized as the
\emph{regret} of the learner:
\[
	\Rc_T =
	\textstyle\sum_{t=1}^T \loss_{t,I_t} - \min_{i\in [k]} \sum_{t=1}^T
	\loss_{t,i}.
\]
A prediction strategy satisfying $\limsup_{t\to\infty} \Rc_t/t
\leq 0$, is called a \emph{no-regret} algorithm.

If the stream is generated by sampling $(\sample_t,c_t)$ \iid\ from a fixed
distribution, it is called a \emph{stochastic stream}, otherwise we call it
\emph{adversarial}, as if an oblivious adversary has chosen the stream for the
learner. It is known \citep{hazan2019introduction} that if the learner follows
a deterministic strategy, she can be forced by the adversary to have linear
regret. Hence, the learner should randomize and select $I_t \sim \dist_t$,
where $\dist_t$ is some distribution over the experts, reflecting how good the
learner thinks the experts are at round $t$. In this case, we are interested in
the \emph{expected regret} $\ex\brk{\Rc_T}$, where the expectation is \wrt\ the
(possible) randomness in the stream, as well as the randomness of the learner.
By the tower property of expectation, $\ex[\loss_{t,I_t}] = 
\ex[\ex[\loss_{t,I_t}\mid \dist_t, \loss_t]] =
\ex \inner{\dist_t, \loss_t}$, and hence, we could write 
\[
    \ex[\Rc_T] = \ex\big[\textstyle\sum_{t=1}^T \inner{\dist_t,\loss_t} - \min_{i\in [k]} \sum_{t=1}^T
    \loss_{t,i}\big].
\]

On top of the preceding task, it is often desirable that at each round $t$, the
learner \emph{recommends} (or outputs) an expert $\pi_t$ as the best expert so far. This
recommendation is suited for model selection tasks, where one needs not only the
predictions per round, but also a recommendation about which classifier is the
best one.  We measure the quality of $\pi_t$ in two ways: the probability of
returning the true best model of the stream so far (\emph{identification
probability}), and the gap between the accuracy of the recommended model and
the best one (\emph{accuracy gap}). The choice of measure depends on the
application: if one is interested only in identifying the best model, then the
first measure, and if one just cares about getting a model that has an accuracy
close to the best classifier, then the second measure is more relevant.

\section{ALGORITHM AND ANALYSIS}
In this section we set up the notation and present the \mpick\ algorithm, along
with several theoretical results regarding its performance.
\subsection{The Algorithm}\label{sec:the-algorithm}

At any round $t$, our algorithm, based on the predictions $\pred_t$ and current
distribution $\dist_t$ decides to query the label with probability $\query_t$ (to be
determined later). 
Let $\queryind_t \sim \mathrm{Ber}(q_t)$ be the indicator of querying.
Our algorithm then constructs a \emph{loss estimate} 
$\ellest_t = \loss_t/\query_t\cdot \queryind_t$.
With this trick, we can think that the learner observes the loss sequence
$\crl{\ellest_t}_{t\geq 1}$. We then construct $\dist_t$ similar to the
Exponential Weights (EW) algorithm \citep{littlestone1989weighted} for this
loss estimate sequence and with decaying
learning rates $\crl{\eta_t}$.
The detailed algorithm is depicted in \pref{alg:algorithm}. In what follows, we
also set $\Lest_t = \sum_{s\leq t} \ellest_s$ and $\Loss_t = \sum_{s\leq t}
\loss_s$.

\begin{algorithm}[t]
\SetAlgoLined
Set $\Lest_{0,i}=0$ for all $i \in [k]$

\For{$t=1,2,\ldots$}{

$\eta_t \coloneqq \sqrt{(\log k)/(2t)}$

Compute the distribution $\dist_t$ over models, with $\dist_{t,i} \propto \exp \{-\eta_t \Lest_{t-1, i} \}$

Get predictions $\pred_t$ of models for the observed data instance $\sample_t$

Recommend $\pi_t := \argmax_{i\in [k]} \dist_{t,i}$ as the best model up to round
$t$

Sample $I_t \sim \dist_t$ and output $\pred_{t,I_t}$ as the predicted label for this instance

Compute $\query_t$ as in \pref{eq:definition-update-probability} and sample $\queryind_t\sim \text{Ber}(\query_t)$

\eIf{$\queryind_t=1$}{
Query the label $c_t$

$\Lest_{t, i} = \Lest_{t-1, i} +\frac{1}{\query_t}\indic{\pred_{t, i}\neq c_t}$, $\forall i \in [k]$
}
{
    $\Lest_{t,i} = \Lest_{t-1,i},\quad \forall i \in [k]$
}
}
\caption{\mpick}
\label{alg:algorithm}
\end{algorithm}

\paragraph{Query Probability} 
Instead of observing $c_t$ with a constant probability (as done by
\citet{cesa2005minimizing}), we {\em adaptively} set this probability according
to the predictions $\pred_t$ and our current distribution over the experts
$\dist_t$.
Notice that, based on the predictions, we know that the true loss vector
$\loss_t$ is among $\crl{\loss_t^{c} : c\in \Cc}$, where $\loss_t^{c}$  is
the hypothetical loss vector if the true label was $c$, \ie, $\loss_{t,i}^c =
\indic{\pred_{t,i}\neq c}$. 
We define
\begin{align*}
    \querybef(\pred_t, \dist_t) &= \max_{c\in \Cc} \mathop{\mathrm{Var}}_{J \sim \dist_t}
            \loss^{c}_{t,J} \nonumber \\
        &= \max_{c\in\Cc} \inner{\dist_t,\loss^{c}_t}(1-\inner{\dist_t,\loss^{c}_t})
\end{align*}
to be the maximum possible variance among different possible losses \wrt\
the distribution $\dist_t$, and we set
\begin{align}\label{eq:definition-update-probability}
	\query_t = \begin{cases}
        \max\crl{\querybef(\pred_t, \dist_t), \eta_t} & \text{if }\querybef(\pred_t, \dist_t) \neq 0 \\
		0 & \text{otherwise.}
		\end{cases}
\end{align}
When $\querybef(\pred_t,\dist_t)$ is nonzero, as seen above, we utilize a lower bound
on $\query_t$ to prevent unboundedness issues. This lower bound, however, decreases
over time.

The intuition behind the definition of $\querybef(\pred_t, \dist_t)$ is as follows.
Hypothetically, if the true label is $c$ and we observe it, the distribution
$\dist_t$ over the models would be updated to $\dist_{t+1}$ according to the loss
$\loss_t^{c}$. If we miss this update, as shown in
\pref{app:var-close-to-kl}, the amount of regret we accumulate (due to
not updating $\dist_t$ to $\dist_{t+1}$) is proportional to the variance of
$\loss_t^{c}$.
Hence, the maximum variance among all hypothetical losses is a measure of
the importance of the instance at hand and we use this value in our query
probability. Note that if all models make the same prediction,
observing the true label has no effect on the regret, and this behaviour is
also reflected in \pref{eq:definition-update-probability}, as $\query_t$ would be
equal to zero in this case.

In what follows, we first tackle the general case of adversarial streams and
prove bounds on regret, number of queries, and accuracy gap of \mpick. We then
strengthen our results for the stochastic setting and give improved bounds as
well as a bound for the identification probability. All omitted proofs can
be found in \pref{app:proofs}.

\paragraph{Notation} In what follows we define the conditional expectation
$\ex_t[\cdot] := \ex[\cdot\mid \Fc_{t-1}]$, where $\Fc_{t-1}$ is the
$\sigma$-algebra generated by all the random variables up to and including time
$t-1$. Moreover, denote by $a\vee b := \max\crl{a,b}$. We use $\inner{\cdot,\cdot}$
to denote the inner product of vectors. For a label $c\in\Cc$, we set
$\dist_{t,c} := \inner{\dist_{t}, \loss_t^c}$.

\subsection{Guarantees for Adversarial Streams}
We first prove that our algorithm has no regret. It is known
\citep{cesa1997use} that the regret of any online algorithm that observes all of
the labels is at least $\Omega(\sqrt{T\log k})$. Our regret bound matches this
lower bound, even though we do not see all the labels.  Compared to LEP, our
regret bound is smaller: they prove that for a fixed query probability
$\epsilon$, the regret is
bounded by $\sqrt{2T\log k / \epsilon}$, and  for getting a regret of
$O(\sqrt{T\log k})$ one has to set $\epsilon$ to be a constant. This forces the
number of queries to be linear in $T$.
However, there are no additional terms in our regret bound, as the
probability of querying is adapted to the stream.
\begin{theorem}[Regret]\label{thm:main-regret}
For adversarial streams, the expected regret of \pref{alg:algorithm} is bounded
above by
\[
	\ex\brk{\Rc_T} \leq 2 \sqrt{2T\log k}.
\]
\end{theorem}
\begin{proof}
    We bring a few important observations that help us in the proof. Observe
    that we can remove those rounds where $\loss_{t,i}=1, \forall i\in[k]$,
    since expert $I_t$ and the best expert in hindsight make the same
    mistake at round $t$. In the remaining rounds, $\ellest_t$ has the same
    conditional expectation as $\loss_t$:
    \[
        \ex_t[\ellest_t] \stackrel{\mathrm{(a)}}{=} \ex_t[\ex_t[\ellest_t\mid \pred_t,\loss_t]]
        \stackrel{\mathrm{(b)}}{=} \ex_t[\tfrac{\loss_t}{q_t}\ex_t[Q_t\mid
        \pred_t,\loss_t]] = \loss_t,
    \]
    where (a) is by the tower property of expectation and (b) is by the definition
    of $\ellest_t$ and the fact that $q_t$ is
    $\sigma(\pred_t,\Fc_{t-1})$-measurable.
    This, together with 
    $
        \textstyle\ex\brk{\min_{i\in [k]} \Lest_{T,i}} \leq
	\min_i\ex\brk{\Lest_{T,i}} \leq \min_i \Loss_{T,i},
    $
    immediately implies that the expected regret of the algorithm for the
    loss sequence $\crl{\loss_t}$ is upper bounded by the expected regret for
    $\crl{\ellest_t}$. Hence, in what follows, we bound the expected regret for
    the latter.

    The expected regret can be decomposed as
    \[
        \ex[\Rc_T] = \ex(\textstyle\sum_{t} m_t - \Lest_{T,*}) +
        \textstyle\ex\sum_t\ex_t[\inner{\dist_t, \ellest_t} - m_t],
    \]
    where $m_t := -\eta_t^{-1}\log \inner{\dist_t, e^{-\eta_t\ellest_t}}$ is the
    mix loss and $\Lest_{T,*} := \min_i \Lest_{T,i}$. Bounding the first part is
    standard and by \pref{lem:mixloss}, it is at most $\log k / \eta_T$. For
    the second term in the regret decomposition, we show in \pref{lem:mixinggap}
    that the $t$\textsuperscript{th} term in the sum is bounded by $\eta_t$. Our proof of this
    lemma heavily relies on how we defined $\querybef(\pred_t,\dist_t)$ and
    the form of our estimated losses.
    Plugging in $\eta_t = \sqrt{\log k / 2t}$ finishes the proof.%
\footnote{\label{rem:on-omd}The attentive reader familiar with OMD/FTRL might
	have realized that the proof deviates from the usual proof methods. In a nutshell,
	if we consider a general regularizer, following the usual proofs, one has to
	bound the stability of the algorithm, which boils down to bounding
	$\nrm{\ellest_t}_{t,*}^2$ by a constant, where $\|\cdot\|_t$ is the local
	norm at round $t$ induced by the inverse Hessian of the regularizer. As
	$\ellest_t \in [0, 1/\eta_t]$, it can scale up to $O(\sqrt{T})$ and there is
no trivial way to bound the norm, as the norms are equivalent in $\Rbb^k$.}
\end{proof}

Our next result concerns the number of queries. We show that in the adversarial
setting, this number depends linearly on the total mistakes of the best model
(not taking into account the rounds where all models misclassify the instance). For
example, if the best model is perfect, the query count is $O(\sqrt{T})$.

\begin{theorem}[Queries]\label{thm:nquery-adv}
    Assume that in every round there are at least two models that disagree. Also
    assume that the total number of mistakes of the best model satisfies
    $\Loss_{T,*} \leq (\frac{|\Cc|-1}{|\Cc|}- \epsilon)T$ for some $\epsilon>0$.
    Then, for $T \geq 4\log k /\epsilon^2$, the expected number of queries up to
    round $T$ is at most $5\sqrt{T\log k}+ 2\Loss_{T,*}$. 
\end{theorem}
\begin{proof}
    The main idea is to relate the number of updates to
    the regret. First, we bound $\query_t$ from above by $\eta_t + \sum_{c\in
    \Cc}\dist_{t,c}(1-\dist_{t,c})$, as maximum is smaller than the sum. Then,
    using the concavity of $a(1-a)$ and Jensen's inequality we further  bound the
    sum over classes by $r_t\prn{2 - \tfrac{|\Cc|}{|\Cc|-1}r_t}$, where $r_t =
    \inner{\dist_t, \ell_t}$.  The proof finishes by summing over
    $t$ and carefully invoking Jensen's inequality again.
\end{proof}
\begin{remark}
    If all models are bad (\ie, if $L_{T,*} \approx T$), then our algorithm
\emph{can} query a lot, and the bound above is not loose. A simple adversarial
example is illustrated in \pref{app:example-large-updates}. Better
bounds on the number of queries are possible with more assumptions on the
stream, e.g., when the stream is stochastic.
\end{remark}

\looseness -1 We now consider the quality of \mpick's recommendations for model
selection.  In the full generality of the adversarial setting, one cannot say
much about the identification probability. However, if we restrict the adversary
and assume that after some round $t_0$, the cumulative loss of the models
start to deviate and keep a minimal gap, we can give a sharp lower bound
on the identification probability, as well as a stronger bound on accuracy gap.
We call an adversary $(t_0, \Delta)$-\emph{restricted} if there exists some
expert $i^* \in [k]$ so that for all $t\geq t_0$, $\Loss_{t,j} \geq
\Loss_{t,i^*} + \Delta t$ for all $j\neq i^*$.

If the algorithm recommends $\pi_t$ at round $t$, its accuracy gap is defined as
$\frac{1}{t}(\Loss_{t, \pi_t} - \Loss_{t,i^*})$ and its identification
probability is $\Pr\crl{\pi_t = i^*}$, where $i^*$ is the best model up to round
$t$, \ie, $i^* = \argmin_{i\in[k]} \Loss_{t,i}$ (notice the use of $\Loss$
instead of $\Lest$ in both definitions).

\begin{theorem}[Accuracy Gap]\label{thm:acc-gap-adv}
Under no assumptions on the adversary, modify the algorithm to recommend $\pi_t
= I_\tau$, where $\tau\in[t]$ is selected uniformly at random. Then, to reach
an expected accuracy gap of at most $\epsilon$, it is enough to have 
$
	t \geq 8\cdot\log k/ \epsilon^2.
$

Moreover, if the adversary is $(t_0, \Delta)$-restricted, by recommending
$\pi_t = \argmax_{i\in[k]} \dist_{t,i}$ and 
\[
	t \geq \min\crl*{31\cdot\frac{\log k}{\Delta^4} \log^2(\tfrac{1}{\epsilon}),
	t_0},
\]
one gets an expected accuracy gap of at most $\epsilon$.
\end{theorem}
The proof of the first part is based on our regret bound and is standard. The
second part is a simple corollary of \pref{thm:iden-adv} below.
The difference between the two guarantees is twofold: while the first guarantee
is instance independent, its dependence on $1/\epsilon$ is quadratic. However,
the second guarantee comes with poly-logarithmic dependence on $1/\epsilon$, but with
an instance-dependent constant $1/\Delta^4$.
\begin{theorem}[Identification Probability]\label{thm:iden-adv}
    If the adversary is $(t_0,\Delta)$-restricted, the probability that we
    misidentify the best model at round $T\geq t_0$ is at most
    \[
        \Pr\crl*{\pi_T \neq i^*} \leq k\cdot e^{-0.18\Delta^2\sqrt{T\log k}}
    \]
\end{theorem}
This theorem, together with \pref{thm:ident-stoch} below, clearly shows why \mpick\ is
successful in model selection tasks, as the probability of misidentifying the
best model decreases (close to) exponentially fast, even if the stream is (restricted)
adversarial. The proof is similar to \pref{thm:ident-stoch} and is based on
martingale arguments.

\subsection{Guarantees for Stochastic Streams}
In this section, we assume that the stream is \iid\ and provide stronger
results.  Let $i^* \in [k]$ be the model with the highest expected accuracy, and
define $\Delta_j = \ex\brk{\loss_{\cdot, j} - \loss_{\cdot, i^*}}$ for all
$j\in[k]$ to be the gap between the accuracies of model $j$ and the best model.
Also define $\theta_j = \Pr\crl{\loss_{\cdot, j} \neq \loss_{\cdot, i^*}}$ to be
the probability that exactly one of $j$ and $i^*$ correctly classify a sample.
Define
\[ 
    \lambda = \min_{j\in [k]\setminus \{i^*\}} \Delta_j^2 / \theta_j.
\]
Intuitively, $\lambda$ measures the hardness of the instance for our algorithm.
Set $\Delta = \min_{i\neq i^*}\Delta_i$ and assume that $\Delta > 0$ (\ie,
there is a unique best model). To simplify the exposition, we always assume,
w.l.o.g., that in all rounds at least two models disagree, as the rounds in
which all models agree do not contribute to the regret or to the number of
queries. The \emph{pseudo-regret} is defined as $R_T = \ex\sum_t
\inner{\dist_t, \loss_t} - T\Delta$.

We first improve \pref{thm:nquery-adv} and show on average \mpick\ asks
$O(\sqrt{T\log k}\cdot |\Cc|/\Delta)$ labels. The dependence on $1/\Delta$ has
the following intuition: it takes on average $1/\Delta$ rounds to observe an
instance where the best model performs better than the rest. The bound shows
that \mpick\ needs no more than $O(\sqrt{T\log k})$ of these instances to build
up sufficient confidence in the best model.
\begin{theorem}[Queries]\label{thm:nquery-stoch}
The expected number of queries up to round $T$ is bounded by
\[
	\ex\brk*{{\textstyle \sum_{t=1}^T \queryind_t}} \leq \sqrt{2T\log k}(1 + 4\tfrac{|\Cc|}{\Delta}).
\]
\end{theorem}
\begin{proof}
\looseness -1 Notice that the expected regret is lower bounded by the
pseudo-regret and upper bounded by our adversarial regret bound
(\pref{thm:main-regret}). These bounds imply 
\[
	\ex\textstyle\sum_{t=1}^T (1-\dist_{t,i^*})\Delta \leq R_T \leq \ex\brk{\Rc_T} \leq 2\sqrt{2T\log k}.
\]
Hence, $\ex\sum_{t=1}^T(1-\dist_{t,i^*}) \leq \frac{2\sqrt{2T\log k}}{\Delta}$.
This means that $\dist_{t,i^*} > \frac{1}{2}$ most of the times: if $N$ is
the number of rounds such that $\dist_{t,i^*} \leq \frac{1}{2}$, we have
\[
	\frac{1}{2}\ex N  \leq \ex\sum_{t=1}^T(1-\dist_{t,i^*}) \leq  \frac{2\sqrt{2T\log k}}{\Delta}.
\]
Now, by the definition of $\query_t$ we have
$
	\query_t \leq \eta_t + \sum_{c\in \Cc} \dist_{t,c}(1- \dist_{t,c}).
$
For a class $c\in\Cc$ that is present among the models predictions at round $t$, we can write
$
	\dist_{t,c}(1- \dist_{t,c}) = (\dist_{t,i^*} + a)\cdot b,
$
for some $a,b\geq 0$ with $b\leq 1-\dist_{t,i^*}$. When $\dist_{t,i^*} \geq
\tfrac{1}{2}$, we have $\query_t \leq b \leq 1 - \dist_{t,i^*}$. If $\dist_{t,i^*} \leq
\tfrac{1}{2}$ we bound $\query_t$ by $\tfrac{1}{4}$. Summing over $t$ and using the bound on $N$, we
finishes the proof.
\end{proof}
The next three results are parallel to the ones in the previous section. By
adopting careful martingale arguments, we first show that the probability of
misidentifying the best model decreases (close to) exponentially with a rate
depending on $\lambda$.

\begin{theorem}[Identification Probability]\label{thm:ident-stoch}
For $T > 2\log k$, the probability that we misidentify the best model at round
$T$ is at most
\[
	\Pr\crl*{\pi_T \neq i^*} \leq k\cdot e^{-0.18\lambda \sqrt{T\log k}}.
\]
\end{theorem}
\begin{proof}[Proof Sketch]
	Notice that $\xi_t = \Delta_j - \ellest_{t,j} + \ellest_{t,i^*}$ is a martingale
	difference sequence. Using a variation of Freedman's inequality for
	martingales and a careful analysis, one arrives at the theorem. 
\end{proof}
Bounds on accuracy gap follow easily. The idea is that by
\pref{thm:ident-stoch}, the best arm is always recommended, except for a
constant number of rounds.
\begin{theorem}[Accuracy Gap]\label{thm:acc-stoch}
For
\[
	T \geq 31\cdot\frac{\log^2(k\max_i\Delta_i)}{\lambda^2 \log
	k}\log^2(\tfrac{1}{\epsilon}),
\]
recommending $\pi_T$ results in an expected accuracy gap of at most $\epsilon$.
\end{theorem}
To bound the regret, \pref{thm:main-regret} is still applicable. Additionally,
if one predicts according to $I_t = \pi_t$ (a.k.a. Follow The Leader strategy),
the following theorem shows that the pseudo-regret is bounded by a {\em constant}.%
\footnote{In full information, when one observes all the
labels, the FTL strategy fails to have the no-regret property in the adversarial setting.
However, it has been shown that it favors a constant regret bound in stochastic
settings. We show that our algorithm has the same behaviour.}
\begin{theorem}[Regret]\label{thm:regret-stoch}
If in \pref{alg:algorithm} one sets $I_t = \pi_t$ for all $t$, then the
pseudo-regret bounded by a constant: 
\[
	R_T \leq 62\max_i \Delta_i\tfrac{k}{\lambda^2 \log k}.
\]
\end{theorem}


\section{EXPERIMENTS}\label{sec:experiments}

\begin{figure*}[t!]
    \centering
      \includegraphics[width=1\linewidth]{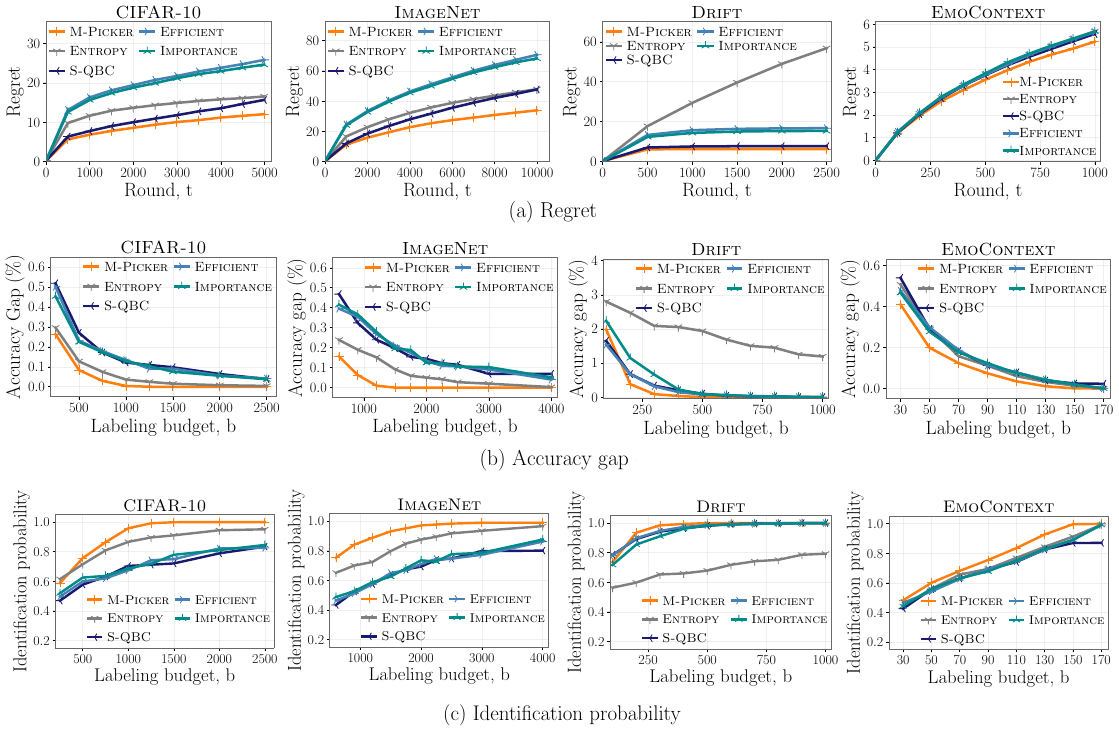}
        \caption{Performance of $\mpick$ (\textsc{M-Picker}) and other adapted baselines on four datasets \{\cifar, \imagenet, \drift, \emo\}.
        \mpick\ is able to output the true best model with high probability, while querying up to 2.6$\times$ fewer labels than the best competing method.}
        \label{fig:combined-plots}
        \vspace{-1em}
\end{figure*}
\looseness -1 We conduct an extensive set of experiments to demonstrate the
practical performance of \mpick\ for online model selection and 
sequential label prediction. We first run experiments on common data sets where
the instances come \iid\ from a fixed data distribution. This setting allows us
to empirically assess the performance in the stochastic setting. We then
consider a more challenging scenario where examples come from a drifting data
distribution, which we treat as an adversarial stream.
\paragraph{Datasets and Model Collection}
We conduct our experiments using various models trained on common datasets such
as the SemEval 2019 dataset (\emo) for emotion detection~\citep{Semeval2019} and the long-term gas sensor drift dataset (\drift) from the UCI Machine
Learning Repository~\citep{uci2012sensordrift, vergara2012chemical} as well as on more
complex datasets of natural images such as \cifar\ and
\imagenet. These datasets cover a wide range of scale:
\cifar, \emo\ and \drift\ are of smaller scale while \imagenet\ is a large scale
dataset. 
Each dataset consists of a large test set (which we later use to construct streams of
examples) and (possibly multiple) training sets. For each dataset, we collect
a collection of pre-trained models by training various models on the training sets.
We provide a detailed explanation on the characteristics of our model
collection in \pref{app:dataset-explained}.

For \cifar, we trained 80 classifiers varying in model, architecture, and
parameter settings available on Pytorch Hub\footnote{\url{https://pytorch.org/hub/}}.  The ensemble contains models having accuracies between
55-92\% on a test set consisting of 10\,000 \cifar\ images. The
\imagenet\ dataset poses a 1\,000-class classification problem. We collected 102
image classifiers that are available on TensorFlow Hub\footnote{\url{https://tfhub.dev/}}. The
accuracy of these models is in the range 50-82\%. For the test set, we use the
whole official test set with 50\,000 images.  For the \emo\ dataset, we
collected 8 pre-trained models that are the development history of a participant
in SemEval 2019. The accuracy of the models varies in 88-92\% on a test
set of size 5\,509. Lastly, for the \drift\ dataset, we trained an SVM
classifier on each of 9 batches of gas sensor data that were measured in
different months. We use the last batch as a test set, which is of size 3\,000.
Due to the drift behaviour of sensor data among different time intervals, the
accuracy of the models on the test set is relatively low, and lies in
25-60\%. 

\vspace{-0.6em}
\paragraph{Baselines} \looseness -1 To compare with existing selective sampling
strategies, we implement variations of \qbc, namely, vote entropy
(\textsc{Entropy}) and structural \qbc\ (\textsc{S-QBC}) as well as label
efficient prediction (\textsc{Efficient}) and importance weighted active
learning (\textsc{Importance}), as described below. Typically, these methods
follow a coin flipping strategy: upon seeing an instance $x_t$, a coin is
flipped with a bias $\query_t$, and the label $\label_t$ is requested if and only if the
coin comes up heads.

\emph{Label Efficient Prediction/Passive Learning}. We implement
\citep{cesa2005minimizing} by querying the label of each round randomly with a
fixed probability $\query_t=\varepsilon$. For a fair comparison, we restrict our
interest merely to the data instances in which at least two models disagree, as
others are non-informative in the ranking of models. In our evaluation, we set
the query probability to $\epsilon = b/T$ for having an expected number of $b$
queries in a stream of size $T$. Note that our way of setting $\epsilon$ depends
on the whole stream for having comparable results in terms of the number of
queries, as we shall drop the non-informative samples first.

\emph{QBC/Vote Entropy}. We use the method of ~\citet{Dagan1995CommitteBasedSamp}
and adapt it to the streaming setting as a disagreement-based selective sampling
baseline.  Upon seeing each instance, we measure the disagreement between the
model predictions to compute the query probability. In our implementation,
we consider every pre-trained model as a committee member and use vote
entropy as the disagreement measure.

\emph{Structural QBC}. The (interactive) structural \qbc\
algorithm~\citep{Tosh2018Sqbc} is built upon the \qbc\ principle,
and its query probability is specified via the disagreement between competing
models that are drawn from a posterior distribution $\boldsymbol{\rho}_t$.
After each new query, the posterior is updated as $\boldsymbol{\rho}_{t+1} \propto
\boldsymbol{\rho}_{t}\exp(-\beta {\loss}_t)$, where $\beta$ is a fixed parameter.
In our implementation, at each round $t$, we draw two models $i$ and $j$
from $\boldsymbol{\rho}_t$ with replacement and set the query probability to be the fraction of
disagreement between $i$ and $j$ up to round $t$, that is, $\query_{t} =
\frac{1}{t}\sum_{s\leq t} \indic{\pred_{s,i}\neq \pred_{s,j}}$.

\emph{Importance Weighted Active Learning}.
We implemented the importance weighted active learning algorithm introduced
by~\citet{Beygelzimer2008ImpWeightedAL}, as well as its variant for efficient
active learning~\citep{Beygelzimer2010Agnostic, Beygelzimer2011EfficientAL}.
Among these two adaptations of importance weights, we only focus on the
superior~\citep{Beygelzimer2008ImpWeightedAL} in our empirical evaluation and
leave the others to \pref{app:baselines}.

It is crucial to note that none of the methods above are tailored for the task
of \emph{ranking pre-trained models} and (except for \citet{cesa2005minimizing}) for
\emph{sequential label prediction}. Yet we consider them as selective sampling
baselines; see \pref{app:baselines} for further discussions about our baselines.

\begin{figure*}[t!]
    \centering
      \includegraphics[width=1\linewidth]{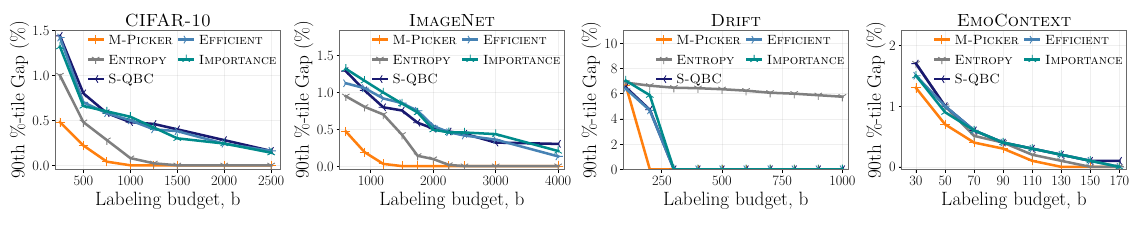}
        \caption{Worst-case analysis on the outputted models: 90th percentile accuracy gap}
        \label{fig:worst-case gap}
        \vspace{-1em}
\end{figure*} 

\subsection{Experimental Setup}
\paragraph{Evaluation Protocol and Tuning} \looseness -1 For a fair comparison,
we focus on the following protocol. We sequentially draw $T$ \iid\ instances
uniformly at random from the entire pool of test instances, then input it into
each algorithm as a stream, and call it a realization. In each realization, the
pre-trained model with the highest accuracy on that stream (considering all
labels) is denoted as the \emph{true} best model of the realization.

For each realization and up to any round $t$, \mpick\ outputs $\pi_t =
\argmax_i \dist_{t, i}$ as the best model, and other methods output the model
having the highest
accuracy on the queried labels. Upon exhausting the stream, we evaluate
the performance of each method based on the model that is outputted. We realize
this process many times to have an estimate of the expected performance.

For comparing the methods under the same budget constraint, we tune the
(hyper-)parameter of each method to query the same number of instances, and
compare their average performance under various labeling budgets. For Structural
\qbc, we treat $\beta$ (in the posterior) as the
hyperparameter. For \qbc\ with vote entropy, importance weighted active learning and
\mpick, we \emph{introduce} a hyperparameter $\beta$ to scale the query probability
according to the given labeling budget.%
\footnote{It is straightforward to see
that by scaling the value of $\querybef(\pred_t, \dist_t)$ by some constant, one
still gets similar theoretical results. The regret bounds, as well as the bounds
on the confidence and accuracy gap will be scaled accordingly.}
Note that by default, \mpick\ needs no hyperparameters, and we introduce $\beta$ 
for the sole reason of fair comparison with other methods.
We perform hyperparameter selection via a grid search. The hyperparameters used
for each budget, together with a large range of hyperparameters and their
respective budgets can be found in \pref{app:hyperparameters}.

\vspace{-0.6em}
\paragraph{Performance Metrics}
For a given labeling budget, we consider the following key quantities as
performance measures: \emph{Regret} for a fixed labeling budget, \emph{Accuracy
gap} between the outputted model and the true best model, and
\emph{Identification Probability}, which is the fraction of realizations that
methods return the true best model of that realization.
\vspace{-0.6em}

\paragraph{Scaling and Computation Cost}
We conduct our experiments on different stream sizes. We choose sizes of 5\,000,
10\,000, 1\,000 and 2\,500 for \cifar, \imagenet, \emo\ and \drift\ test sets,
respectively.  We implement \mpick, along with all other baseline methods in
Python. All the baseline methods combined, each realization takes between 1
second (for \emo) and 4 minutes (for \imagenet) when executed on a single CPU
core. \mpick\ alone takes between 75 miliseconds (for \emo) and 47 seconds (for
\imagenet). For all datasets we run 500 independent realizations for each budget
constraint. To improve the overall runtime, we run the realizations in parallel
over a cluster with 400 cores.

\subsection{Experimental Results}
We review our numerical results for each of the metrics introduced earlier.
We refer to \pref{app:experiments_detailed} for an extensive discussion of our
findings.  For each of our metrics, we observe the following:
\vspace{-0.6em}

\paragraph{Regret}
We measure the regret across all rounds and for those budgets where \mpick\
returns the best model with high confidence. Namely, we set the budget to
1\,250, 1\,200, 130 and 1\,000 for the \cifar, \imagenet, \emo\ and \drift\
datasets, respectively.  The regret behaviour is shown in
\pref{fig:combined-plots}(a). In all cases, the regret grows sub-linearly for
all algorithms. The regret of our algorithm in all cases is smaller up to a
factor of 1.3$\times$, which shows that \mpick\ can be used for sequential label
prediction tasks as well as model selection.

\vspace{-0.6em}
\paragraph{Accuracy Gap} Next, we consider the average accuracy gap over the
realizations.  \pref{fig:combined-plots}(b) shows that the accuracy gaps for
\mpick\ are much smaller than that of other adapted methods under the same
budget constraints. Quantitatively, in both \cifar\ and \imagenet\ datasets,
\mpick\ achieves the same expected accuracy gap as \entropy\ by querying nearly
2.5$\times$ less labels. For the \drift\ dataset, for instance, \mpick\ returns
a model that is within a 0.1\%-neighborhood of the accuracy of best model after
querying merely 11\% of the entire stream of examples (when the budget is 270
for a stream of size 2\,500). Note that active learning over drifting data
distribution is a very challenging task, and \efficient\ (Label Efficient Prediction/Passive
Learning) is considered the strongest baseline~\citep{settles2009active}.
Our experiments thus suggest that, even for small labeling budgets, \mpick\
returns a model whose accuracy is close to that of the best model, if not the best model itself.


\vspace{-0.6em}
\paragraph{Identification Probability} As illustrated in
\pref{fig:combined-plots}(c), \mpick\ achieves significant improvements of up to
2.6$\times$ in labeling cost while returning the true best model and requesting
far fewer labels than other adapted methods. For \cifar, \imagenet, \emo\ and
\drift\ datasets, \mpick\ queries 2.5$\times$, 2.5$\times$, 1.2$\times$ and
1.7$\times$ fewer labels respectively than that of the best competing method
(mainly \entropy) to reach confidence levels 95\%, 97\%, 92\% and 97\%,
respectively. This shows that \mpick\ is able to achieve the same identification
power as the adapted baselines at a much lower labeling cost.

\subsubsection{On the Robustness of \mpick}
Practitioners are often interested in the relative quality of the output model
compared to the true best model in a single trial. In this regard, and in the
spirit of Theorems~\ref{thm:iden-adv} and \ref{thm:ident-stoch}, we conduct
further numerical analysis on the accuracy of the outputted models over a large
number of realizations to investigate if \mpick\ performs well \emph{with high
probability}. We compute the 90th percentile of accuracy gap as a proxy for
the behaviour of the algorithms in the high probability regime
(see~\pref{fig:worst-case gap}). In the \drift\ dataset, for instance, \mpick\
returns the true best model after querying merely 8\% of the labels (when the
budget is 200 with a stream size of 2\,500). For the \cifar\ and \imagenet\
datasets, \mpick\ returns a model that is within a 0.1\%-neighborhood of the
accuracy of best model after querying nearly 12\% of the entire stream of
examples whereas the best competing method achieves this after querying 24\% of
the same stream of examples. Moreover, \mpick\ outputs the true best model after
querying 15\% and 20\% of the entire stream of examples, respectively. These
results clearly demonstrate the robustness of \mpick.


\section{CONCLUSIONS}
\looseness -1 We introduced an online active model selection approach --\mpick--
to selectively query the labels of instances that are informative for ranking
pre-trained models and to sequentially predict unseen labels. Our framework is generic,
easy to implement, and applies across various classification tasks.  We derived
theoretical guarantees and illustrate the effectiveness of our method on several
real-world datasets.

\subsection*{Acknowledgements}
We thank the reviewers for their constructive feedback. This research was supported
by the SNSF grant 407540\_167212 through the NRP 75 Big Data program and by the
European Research Council (ERC) under the European Union’s Horizon 2020 research
and innovation programme grant agreement No 815943. CZ and the DS3Lab gratefully
acknowledge the support from the Swiss National Science Foundation (Project
Number 200021\_184628), Innosuisse/SNF BRIDGE Discovery (Project Number
40B2-0\_187132), European Union Horizon 2020 Research and Innovation Programme
(DAPHNE, 957407), Botnar Research Centre for Child Health, Swiss Data Science
Center, Alibaba, Cisco, eBay, Google Focused Research Awards, Oracle Labs,
Swisscom, Zurich Insurance, Chinese Scholarship Council, and the Department of
Computer Science at ETH Zurich.

\bibliographystyle{plainnat}
\bibliography{online}

\newpage

\onecolumn
\appendix
\aistatstitle{Online Active Model Selection for Pre-trained
Classifiers:\\Supplementary Materials}

\startcontents[sections]
\printcontents[sections]{l}{1}{\setcounter{tocdepth}{2}}
\vfill
\section{Proofs and Supplementary Lemmas}
\label{app:proofs}
\subsection{On the Choice of Query Probability}\label{app:var-close-to-kl}
Here we elaborate on the discussion for \pref{eq:definition-update-probability}. Let $\dist_t$ be the current distribution over experts, and $\dist_{t+1}^{(c)}$ be the hypothetical distribution over the experts having observed the loss if the true label is $c$.

First, according to \cite[Lemma 1]{narayanan2010random}, the divergence $\mathrm{KL}(\dist_t\|\dist_{t+1}^{(c)})$ accumulates into the regret. This means that if we do \emph{not} update $\dist_t$ accordingly, we \emph{miss} this amount of information.

Second, the KL divergence between $\dist_t$ and $\dist_{t+1}^{(c)}$ computes
\[
    \mathrm{KL}(\dist_t\|\dist_{t+1}^{(c)}) = \eta r + \log(r e^{-\eta} + 1- r),
\]
where $r = \inner{\dist_t, \loss_t^{c}}$. By H\"offdings inequality, one can show that this quantity is between 0 and $\eta^2/8$. As the variance is between 0 and $1/4$, it makes sense to scale the KL divergence to the range $[0, 1/4]$, by multiplying it by $2/\eta^2$. The following lemma completes the comparison promised in \pref{sec:the-algorithm}. We drop the subscript $t$ for readability.

\begin{lemma}
For a distribution over the experts $\dist$ and a fixed $\loss \in \crl{0,1}^k$,
define $\dist_+ \propto \dist\,e^{-\eta \loss}$. Then
$$
\abs*{\mathop{\mathrm{Var}}_{A \sim \dist} \loss_A  - \frac{2}{\eta^2}\mathrm{KL}(\dist\|\dist_+)} < \frac{1}{18\sqrt{3}}\eta + O(\eta^2).
$$
\end{lemma}
\begin{proof}
    Define $r = \inner{\dist, \loss}$. We have that
\[
    \mathrm{KL}(\dist\|\dist_+) =  \eta r + \log(r e^{-\eta} + 1- r) = \log\ex{e^{-\eta(X - \ex{X})}},
\]
where $X$ is a Bernoulli random variable with $\ex{X} = r$. Note that the equation above is the cumulant generating function of $X$ and has the Taylor series
\[
\log\ex{e^{-\eta(X - \ex{X})}} = \frac{\eta^2}{2} \mathop{\mathrm{Var}}(X) + \frac{\eta^3}{6}\kappa_3 + O(\eta^4),
\]
where $\kappa_3$ is the third cumulant. Note that by the relation between cumulants of a Bernoulli random variable, we have
\[
\kappa_3 = r(1-r) \frac{d}{dr} \kappa_2 = r(1-r)(1-2r).
\]
Easy algebra finds that for $r\in [0,1]$, we have $\kappa_3 \in [-1/6\sqrt{3}, 1/6\sqrt{3}]$.

Summing all up, we find
\[
\frac{2}{\eta^2} \mathrm{KL}(\dist\|\dist_+) = \mathop{\mathrm{Var}}(X) + \frac{\kappa_3}{3}\eta + O(\eta^2),
\]
and the result of the lemma follows.
\end{proof}

\subsection{Mix Loss Properties}
\label{app:proof-mix-loss}
Define $\Lest_{T,*} := \min_{i\in [k]} \Lest_{T,i}$.
\begin{lemma}\label{lem:mixloss}
	The cumulative mix loss $M_T$ is bounded above by $\Lest_{T,*} + \frac{\log k}{\eta_T}$.
\end{lemma}
Before stating the proof, first we bring a standard lemma:
\begin{lemma}[{\citet{derooijFollowLeaderIf2013}}]
	The cumulative mix loss $M_T$ has the following properties for constant learning rates ($\eta_t \equiv \eta$ for all $t\geq 1$):
	\begin{enumerate}[label=(\roman*)]
		\item $M_T = -\frac{1}{\eta}\log\sum_{i\in [k]} e^{-\eta \Lest_{T,i}} + \frac{1}{\eta}\log k$,
        \item $M_T \leq \Lest_{T,*} + \frac{1}{\eta}\log k$.
	\end{enumerate}
	Moreover, for any sequence of decaying learning rates $\crl{\eta_t}_{t\geq 1}$, let $M_T(\crl{\eta_t})$ be the corresponding cumulative mix loss, and set $M_T(\eta_T)$ be the cumulative mix loss for fixed learning rate $\eta_T$. Then, it holds that $M_T(\crl{\eta_t}) \leq M_T(\eta_T)$.
\end{lemma}
\begin{proof}
	Define $W_t = \sum_{i\in [k]}e^{-\eta\Lest_{t,i}}$. For part (i) observe that
	\[
		\inner{\dist_t, e^{-\eta\ellest_t}} = \sum_{i \in [k]}
        \frac{e^{-\eta\Lest_{t-1, i}}}{W_{t-1}} e^{-\eta\ellest_{t,i}} = \frac{W_t}{W_{t-1}}.
	\]
	Hence, $\log \inner{\dist_t, e^{-\eta\ellest_t}} =-\eta m_t = \log W_t - \log W_{t-1}$, and $M_T = \frac{1}{\eta}(\log W_0 - \log W_T)$. Observing that $W_0 = k$ gives (i).

    Noticing that $W_T \geq e^{-\eta \Lest_{T,*}}$ easily implies (ii).

	For the last part of the lemma, first we prove that $M_T$ for constant learning rate $\eta$ is nonincreasing in $\eta$. This is shown by looking at the derivative of $M_T$ with respect to $\eta$ which is equal to
	\[
		\tfrac{1}{\eta^2}\log\sum_{i\in [k]} e^{-\eta \Lest_{T,i}} - \frac{1}{\eta}\frac{\sum \Lest_{T,i}e^{-\eta \Lest_{T,i}}}{\sum e^{-\eta \Lest_{T,i}}} -\tfrac{1}{\eta^2}\log k \leq 0,
	\]
	as $\log\sum_{i\in [k]} e^{-\eta \Lest_{T,i}} \leq \log k$.

	Now we can prove the last part of the lemma.
	\[
		\sum_{t=1}^T m_t(\crl{\eta_t}) = \sum_{t=1}^T M_t(\eta_t) - M_{t-1}(\eta_{t}) \leq \sum_{t=1}^T M_t(\eta_t) - M_{t-1}(\eta_{t-1}) \leq M_T(\eta_T). \qedhere
	\]
\end{proof}
\begin{proof}[Proof of \pref{lem:mixloss}]
By the lemma above, we see that
$
M_T \leq M_T(\eta_T) \leq \Lest_{T,*} + \frac{\log k}{\eta_T},
$
where we lower bounded the sum over the models by the one that corresponds to
$\Lest_{T,*}$.
\end{proof}
\subsection{Lemmas for Regret Bound}\label{app:proof-regret-inequality}
\begin{lemma}\label{lem:mixinggap}
    For any $\eta > 0$ it holds that
	\[
		\ex\brk*{\log\inner{\dist_t, e^{-\eta\ellest_t}} + \eta\inner{\dist_t, \ellest_t}} \leq \eta^2.
	\]
\end{lemma}
\begin{proof}
	If the predictions are all the same, there is nothing to prove, as $\ellest_t \equiv 0$ and the expectation vanishes.

	Let $\dist_c = \sum_{i:\pred_t(i)=c} \dist_{t,i}$ and set $c^*$ to be the true label of
	this round. We can then rewrite $\query_t$ as $\query_t = \max_{c\in \Cc} \dist_c(1-\dist_c) \vee \eta$. Observe that
    \begin{align*}
        \log\inner{\dist_t, e^{-\eta\ellest_t}} = \log\brk*{ (1 - \dist_{c^*})\exp\crl*{-\tfrac{\eta}{\query_t}\queryind_t} + \dist_{c^*}}.
    \end{align*}
	It is clear that $\ex{\inner{\dist_t, \ellest_t}} = \inner{\dist_t, \ell_t} = 1 - \dist_{c^*}$. Hence, the expected value in the lemma is equal to
    \begin{equation}\label{eq:rhsupper}
		\eta (1 - \dist_{c^*}) + \query_t\cdot\log\brk*{ (1 - \dist_{c^*})\exp\crl*{-\eta/\query_t} + \dist_{c^*}}.
    \end{equation}
	Our desired result follows from \pref{lem:regretinequality} by setting $x = 1 - \dist_{c^*}$ and noticing that $\query_t \geq \dist_{c^*}(1 - \dist_{c^*}) \vee \eta$.
\end{proof}

\begin{lemma}\label{lem:regretinequality}
    For all $x\in (0,1)$ and all $\eta \in (0,1]$, defining $u = x(1-x) \vee \eta$, one has
    \[
        f(x,u,\eta) := \eta x + u\log\brk*{x\,e^{-\frac{\eta}{u}}+1-x} \leq \eta^2.
    \]
	Moreover, for fixed $x$ and $\eta$, $f(x,u,\eta)$ is decreasing in $u$ for $u\geq \eta$.
\end{lemma}
\begin{proof}
    Note that the value of the LHS and RHS agree when $\eta = 0$, so we have to prove that for all $\eta \geq 0$, the derivative of the LHS is at most $\eta$. Fix some $x \in [0,1)$. We prove this fact in two cases:

    \underline{Case where $\eta \leq x(1-x)$.} In this case, $u = x(1-x)$. For brevity, define $y := \frac{\eta}{x(1-x)}$. The derivative of $f$ with respect to $\eta$ becomes
\[
    \frac{ (e^y - 1)x(1-x) }{ e^y(1-x) + x },
\]
and we are left with proving
\[
    \frac{ (e^y - 1)}{ e^y(1-x) + x } \leq 2y
\]
    As $0 < y \leq 1$, we have that $1 + y + y^2/2 \leq e^y \leq 1 + y + (e-2)y^2$. Replacing these bounds in the equation above leaves us with proving that
    \[
        \frac{1 + (e-2)y}{1 + (y + y^2/2)(1-x)} \leq 2.
    \]
    For a fixed $y$, the left hand side is increasing in $x$, and hence, it is enough to prove that
    \[
        1 + (e-2)y \leq 2,
    \]
    but this is true as $y \leq 1$ and $e-1 < 2$. Thus, we are done with the proof of this case.

    \underline{Case where $\eta > x(1-x)$.} In this case, $u = \eta$, and $f(x,\eta) = \eta x + \eta \log \brk*{ x e^{-1} + 1-x }$. To prove the claim, we have to show that $x + \log \brk*{ x e^{-1} + 1-x } \leq \eta$, or, as the left hand side does not depend on $\eta$, we shall prove
    \[
        x + \log \brk*{ x e^{-1} + 1-x } \leq x(1-x),
    \]
    or, equivalently,
    \[
        1 - (1-e^{-1})x \leq e^{-x^2},
    \]
	which is proven in \pref{lem:expbound}. Thus, in both cases, we have proved our inequality and we are done with the proof of the first part of lemma.

	We now prove the monotonicity of $f$ with respect to $u$. For that, we show the derivative of $f$ with respect to $u$ is nonpositive. The derivative computes
	\[
		f'(u) = \log\brk*{xe^{-\eta/u} + 1 -x} + \frac{xe^{-\eta/u}\eta/u}{xe^{-\eta/u} + 1 -x}.
	\]
	Define $a= \eta/u$. The equation above is zero for $a= 0$. So it suffices to show that the derivative of above is nonpositive for $ 0\leq a \leq 1$. Computing the derivative w.r.t. $a$ and setting it less than 0 is equivalent to
	\[
		xe^{-a} + 1 - x - xe^{-a} \geq 0,
	\]
	which is true. Hence, we are done.
\end{proof}

\begin{lemma}\label{lem:expbound}
    For all $x \in [0, 1]$ one has $1 - (1-e^{-1})x \leq \exp\crl{-x^2}$.
\end{lemma}
\begin{proof}
    Note that $\exp(-x^2)$ is concave on $[0, \sqrt{1/2}]$ and convex on $[\sqrt{1/2}, 1]$. Also at $x=0$ and $x=1$, both sides are equal.
    Hence, we just have to show that at $x = \sqrt{1/2}$, the right hand side is bigger than the left hand side, which automatically shows the inequality for $x \in [0, \sqrt{1/2}]$, and we have to show that the derivative of the right hand side is smaller than the left hand side at $x = 1$, which automatically shows the inequality for the other half of the interval, as $\exp(-x^2)$ is convex there. For the first part, evaluate
    \[
        \exp(-1/2) - 1 + (1 - 1/e)\sqrt{1/2} = (1 - 1/\sqrt{e})\prn*{\frac{1}{\sqrt{2}}(1 + \frac{1}{\sqrt{e}}) - 1} \geq 0.
    \]
    For the second part, note that $\frac{d}{dx} \exp(-x^2) = -2x \exp(-x^2)$, and at $x=1$ it is equal to
    \[
        -2/e \leq -(1 - 1/e),
    \]
    as $3/e > 1$. Hence, we are done.
\end{proof}

\subsection{Proof of \pref{thm:nquery-adv}}\label{app:proof-thm-nquery-adv}
\begin{proof}
	We assume that at all rounds we have $\query_t >0$, as there is no label request on the rounds that all models predict the same. First observe that
	\[
		\ex{\sum_{t=1}^T \queryind_t} \leq \ex\crl*{\sum_{t=1}^T \eta_t + \textstyle\sum_{c\in C}\dist_{t,c}(1-\dist_{t,c})},
	\]
    as maximum of positive numbers is less than their sum. Next, at round $t$
    suppose that the true label is $c_t$. As $x(1-x)$ is concave and
    $\sum_{c\neq c_t} \dist_{t,c} = 1 -
    \dist_{t,c_t}=\inner{\dist_t,\loss_t}=:r_t$, using Jensen's inequality we
    have
	\begin{align*}
		&\sum_{c\in C}\dist_{t,c}(1-\dist_{t,c}) \\
		&\hspace{1cm}= \dist_{t,c_t}(1-\dist_{t,c_t}) + \sum_{c\neq c_t}\dist_{t,c}(1-\dist_{t,c})\\
		&\hspace{1cm}\leq \dist_{t,c_t}(1-\dist_{t,c_t}) + (1-\dist_{t,c_t})\prn*{1 - \tfrac{1-\dist_{t,c_t}}{|C|-1}} \\
		&\hspace{1cm}=r_t\prn*{2 - \tfrac{|C|}{|C|-1}r_t}.
	\end{align*}
	Using Jensen now for the concave function $x(2-\tfrac{|C|}{|C|-1}x)$, we get
	\[
		\sum_{t=1}^T\sum_{c\in C} \dist_{t,c}(1-\dist_{t,c}) \leq T\cdot\prn*{\tfrac{\sum r_t}{T}}\prn*{2- \tfrac{|C|}{|C|-1}\tfrac{\sum r_t}{T}}.
	\]
	Now observe that if the expected total loss of the best model is $L^*$, by our regret bound in \pref{thm:main-regret} we have
	\[
		\ex{\textstyle\sum r_t} \leq 2\sqrt{2T\log k} + L^*.
	\]
	Also note that for $x \leq \frac{|C|-1}{|C|}$, the function $x(2-\tfrac{|C|}{|C|-1}x)$ is increasing. Hence, for large enough $T$ (as described in the theorem), $\frac{1}{T}\sum r_t \leq \frac{|C|-1}{|C|}$, and we have
	\[
		\ex{\sum_{t=1}^T\sum_{c\in C} \dist_{t,c}(1-\dist_{t,c})} \leq\prn*{2\sqrt{2T\log k} + L^*}\cdot\prn*{2- \tfrac{|C|}{|C|-1}\prn*{2\sqrt{2\log k/T}+L^*/T}}.
	\]
	Noting that $\sum\eta_t \leq \sqrt{2T\log k}$, one obtains the result.
\end{proof}

\subsection{Proof of \pref{thm:ident-stoch}}\label{app:proof-confidece}
\begin{proof}
First, we remind the following martingale inequality, which is an improved version of McDiarmid's:
\begin{lemma}[{\citet{seldinImprovedParametrizationAnalysis2017}}]
	Let $\xi_1, \ldots, \xi_T$ be a martingale difference sequence with respect to the filteration $\crl{\Fc_t}_{t\leq T}$, where each $\xi_t$ is integrable and bounded. Let $M_t := \sum_{s\leq t} \xi_s$ be the associated martingale. Define $\nu_T = \sum_{t\leq T} \ex\crl{\xi_t^2\mid \Fc_{t-1}}$ and $c_T = \max_{t\leq T} \xi_t$. Then for any $\beta, \nu, c > 0$,
	\[
		\Pr\crl*{\prn*{M_T \geq \sqrt{2\nu\beta T} + \frac{1}{3}c\beta T} \wedge (\nu_T \leq \nu) \wedge (c_T \leq c)} \leq e^{-\beta T}.
	\]
\end{lemma}

Remember that the weight of model $i$ at the end of round $t$ is proportional to
$\exp\crl{-\eta_{t+1} \Lest_t}$.
Hence, identifying the best model $i^*$ after round $t$ reduces to the fact that $\Lest_{t,i^*} = \min_{j\in [k]} \Lest_{t,j}$.
The probability of this event not happening can be bounded by a union bound on the models:
\[
	\Pr\crl{\exists j\neq i^* : \Lest_{t,i^*} \geq \Lest_{t,j}} \leq \sum_{j\neq i^*} \Pr\crl{\Lest_{t,i^*} \geq \Lest_{t,j}} = \sum_{j\neq i^*}\Pr\crl{\Dtil_{t,j} \leq 0},
\]
where we define $\Dtil_{t,j} = \Lest_{t,j} - \Lest_{t,i^*}$.
From now on, we focus on a single model $j$ and drop the index $j$ from $\Delta_j$ and $\theta_j$. Set $d_t := \loss_{t,j} - \loss_{t,i^*}$ and define
\[
	\xi_t := \Delta - \frac{d_t\queryind_t}{\query_t}.
\]
Note that $\ex\crl{\xi_t\mid \Fc_{t-1}} = \ex\crl{\ex\crl{\xi_t\mid \loss_t, \Fc_{t-1}}\mid \Fc_{t-1}} = 0$.
Moreover, the following holds:
\[
	\xi_t \leq \Delta + \frac{1}{\query_t} \leq \Delta + \eta_t^{-1},\quad \ex\crl{\xi_t^2\mid \Fc_{t-1}} = \ex\crl*{\tfrac{d_t^2}{\query_t} \mid \Fc_{t-1}} - \Delta^2 \leq \theta\eta_t^{-1} - \Delta^2.
\]
The sum of the conditional variances up to $T$ satisfies
\[
	\sum_{t=1}^T \ex\crl{\xi_t^2\mid \Fc_{t-1}} \leq T\eta_T^{-1}\theta - T\Delta^2 =: \nu
\]
Also, set $c = \Delta + \eta_T^{-1}$.
By lemma above we have
\[
	\Pr\crl*{\Dtil_{T,j}=\sum_{t=1}^T \frac{d_t\queryind_t}{\query_t} \leq T\Delta - \sqrt{2T\nu\beta} - \tfrac{1}{3}c\beta T} \leq e^{-\beta T}.
\]
We will find the largest $\beta$ such that the right hand side of the inequality above becomes positive. As it is a quadratic polynomial in $\sqrt{\beta}$, we should have that
\[
	\sqrt{\beta} \leq \frac{\sqrt{2\nu/T + \frac{4}{3}c\Delta}- \sqrt{2\nu/T}}{\frac{2}{3}c}.
\]
Now we lower bound the right hand side, and write $\gamma := \eta_T^{-1}$ for brevity:
\begin{align*}
	\frac{\sqrt{2\nu/T + \frac{4}{3}c\Delta} - \sqrt{2\nu/T}}{\frac{2}{3}c} &\geq \frac{4\Delta\sqrt{2\nu/T}}{8\nu/T +\frac{4}{3}c\Delta}&&\text{as }\sqrt{x+a}-\sqrt{x}\geq \frac{2a\sqrt{x}}{4x + a}\\
																			&= \frac{3\Delta\sqrt{2\nu/T}}{6\nu/T + c\Delta}\\
																			&= \frac{3\Delta\sqrt{2(\gamma\theta - \Delta^2)}}{6\gamma\theta - 5\Delta^2 + \gamma\Delta}\\
																			&= \frac{3\sqrt{2}}{\sqrt{\gamma}}\frac{\Delta\sqrt{\theta - \Delta^2/\gamma}}{6\theta+\Delta - 5\Delta^2/\gamma}\\
																			&\geq \frac{3}{\sqrt{\gamma}} \frac{\Delta\sqrt{\theta}}{7\theta} && \text{as }\theta\geq \Delta \text{ and } \Delta^2/\gamma \leq \theta/2.
\end{align*}
Hence, we conclude that setting
\[
	\beta = 0.18 \sqrt{\log k} \frac{1}{\sqrt{2T}} \frac{\Delta^2}{\theta}
\]
gives the desired property. The proof follows by plugging in the value of
$\beta$ and taking a union bound over the experts.
\end{proof}

\subsection{Proof of \pref{thm:iden-adv}}
The proof is very much similar to \pref{thm:ident-stoch}. The difference is that
the conditional variance is bounded above by
\[
	\ex\crl{\xi_t^2\mid \Fc_{t-1}} \leq \eta_t^{-1} - \Delta^2,
\]
and the rest of the proof follows by setting $\theta=1$.

\subsection{Proof of \pref{thm:acc-gap-adv}}
The first bound is standard and can be found in
\citep{slivkinsIntroductionMultiArmedBandits2019}. The argument is completed by
noting that the expected accuracy gap will be bounded by
\[
	\frac{1}{T}\ex\brk{\Loss_{T, I_\tau} - \Loss_{T,i^*}} \leq \sqrt{\frac{8\log k}{T}},
\]
and setting the right hand side less than $\epsilon$.

For the second part, we use \pref{thm:iden-adv}. With probability at most $
k\cdot e^{-0.18\Delta^2\sqrt{T\log k}}$ the recommended expert is not the best,
for which its accuracy gap is at most 1, and otherwise, the best expert is
returned, with accuracy gap 0. Combining the two gives the result.

\subsection{Proof of \pref{thm:acc-stoch}}
The proof is very similar to \pref{thm:acc-gap-adv}, with the difference that
here one upper bounds the accuracy gap by $\max_i\Delta_i$ instead of 1.

\subsection{Proof of \pref{thm:regret-stoch}}
First we prove a lemma that help us proving the theorem:
\begin{lemma}
	The expected number of times that the recommendation $\pi_t$
	is not the best model is a constant up to any round and is bounded by
	$
		\frac{62k}{\lambda^2 \log k}.
	$
\end{lemma}
\begin{proof}
	By \pref{thm:ident-stoch}, we know that the probability of not recommending
	the best model at round $t$ is upper bounded by
	$k\cdot e^{-0.18\lambda \sqrt{T\log k}}$. Using integral approximation, one
	finds that $\sum_{t=1}^\infty e^{-a\sqrt{t}} \leq 2/a^2$ for all $a>0$. This
	gives
	\[
		\ex\brk*{\sum_{t=1}^\infty \indic{\pi_t \neq i^*}} = \sum_{t=1}^\infty
		\Pr\crl{\pi_t \neq i^*} \leq \sum_{t=1}^\infty k\cdot e^{-0.18\lambda
		\sqrt{T\log k}} \leq \frac{62k}{\lambda^2 \log k}.\qedhere
	\]
\end{proof}
Using the lemma, over $T$ rounds, we make at most $\frac{62k}{\lambda^2 \log
k}$ mistakes, for which we get at most $\max_i \Delta_i$ added to the regret,
and in other rounds, we make no mistakes, hence no regrets on those rounds.
Adding up gives the result.

\section{Example for Large Number of Updates}\label{app:example-large-updates}
Consider a binary classification scenario with two models. Set the loss sequence
to be $(1, 0), (0, 1), (1, 0), \ldots$, that is, on the odd rounds the second
model is correct and on the even rounds, the first one. One can see that the probability of
querying the label is always $\dist_{t,1}\dist_{t,2} \vee \eta_t$ for all $t$. Hence,
this probability is always near $1/4$, as the models weights are always around
$1/2$. Hence, the total number of queries is linear.

\section{Experiments}
\label{sec:appendix experiments}

\subsection{Details on the Model Collections}\label{app:dataset-explained}
\begin{itemize}
     \item \underline{\smash{\emph{\cifar}}}: As an image classification dataset, we train 80 models on \cifar\ dataset varying in machine learning models (ranging from DenseNet, Resnet to VGG), architecture and parameter setting.  The ensemble of models have accuracies between 55-92\% on a test set consists of 10\,000 instances. 
     
     \item \underline{\smash{\emph{\imagenet}}}: This dataset consists of 102 image classification models (ranging from ResNet, Inception to MobileNet) pre-trained on \imagenet\ that are available on TensorFlow Hub. The accuracy of models occupy the range in 50-80\%. For each model, we obtain the \imagenet\ validation dataset with 50\,000 data examples, and furthermore normalize and resize them according to expected input format for each model, and finally conduct inference on the given model to produce predicted labels. 
     \end{itemize}
    \begin{table}[h!]
\centering
\caption{Datasets characteristics}
\vspace{-3mm}
\scalebox{0.85}{
\begin{tabular}{l|cccc}
\hline
Dataset & \#Classes & \#Instances & \#Models & Accuracy of Models \\ \hline
{\cifar} & 10 &10\,000  &80 &55-92\%  \\
{\imagenet} & 1\,000 & 50\,000 & 102&50-80\% \\ 
{\drift} & 6 & 3\,000 & 9&25-60\% \\ 
{\emo} & 4 &5\,509  &8 &88-92\% \\
{\cifar\ (worse models)} & 10 & 10\,000 &80 &40-70\% \\ \hline
\end{tabular}
}
\label{tab: datasets}
\end{table}
     \begin{itemize}
     \item \underline{\smash{\emph{\drift}}}: For the \drift\ dataset, we trained models on the gas sensor drift data that is collected over a course of three years. The dataset has ten batches, each collected in different months. We trained an SVM classifier on each of the batch but the last one, and use the last batch of size 3\,000 as test set. Although each model has good training accuracy on the batch it is trained on, namely above 90\%, their accuracy on the test set lies in 25-60\%. This is due to the drift behaviour of sensor data among different time intervals.
     
     \item \underline{\smash{\emph{\emo}}}: This dataset consist of pretrained models that are the development history of a participant on \emo task in SemEval 2019. The task aims to detect emotions from text leveraging contextual information which is deemed challenging due to the lack of facial expressions and voice modulations. We treat each development as an individual pretrained model where development stages differ in various word representations including ELMo~\cite{Peters2018Elmo} and GloVe~\cite{Pennington2014Glove}. The dataset consists of 8 pre-trained models whose accuracy varies between 88-92\% on the test set of size 5\,509.
     
     \item \underline{\smash{\emph{\cifarw}}}: We also train a set of models with relatively lower accuracies on \cifar\ and call it \cifarw. The sole purpose behind creating such a collection is to investigate the performance of $\mpick$ on a practical scenario like this. Using similar model architectures to that of \cifar, the pretrained models have accuracy between 40-70\% on a test set of size 10\,000.
\end{itemize}
\begin{figure*}[h!]
    \centering
      \includegraphics[width=1\linewidth]{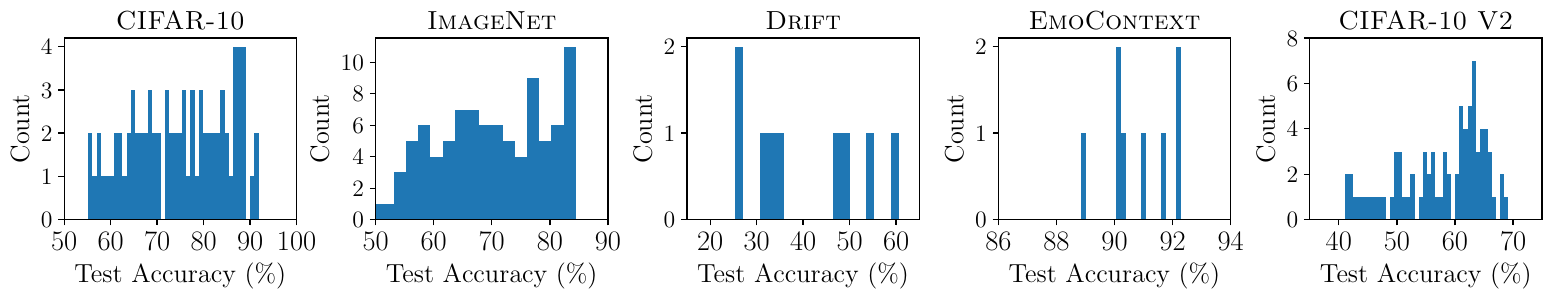}
        \caption{Counts of model accuracies}
        \label{fig:test accuracies}
        \vspace{-1em}
\end{figure*}
The properties of model collections for all datasets are depicted in~\pref{tab: datasets} and~\pref{fig:test accuracies}.

\subsection{Note on the Baselines}
\label{app:baselines}
In this section, we provide further details on some of the adapted baseline methods, namely, query by committee (\entropy), importance weighted active learning (\importance) and efficient active learning (\efal).
\begin{itemize}[leftmargin=*]
   \item \underline{\smash{\emph{Query by Committee}}}: As indicated in
       Section~\ref{sec:experiments}, we adapt the query-by-committee paradigm
       proposed in~\cite{Dagan1995CommitteBasedSamp} for model selection in the
       online setting. The query by committee method consist of two
       sub-strategies (a) ensemble learning, and (b) determining a maximal
       disagreement measure. The ensemble learning indicates how the committee
       is formed from the candidate classifiers. This step is crucial to make
       the disagreement measure more reliable while aiming to form a set of
       classifiers with high accuracy. In literature, there exist many ensemble
       learning methods including~\cite{Mamitsuka1998Boost,
       Melville2004Diverseensembles, Breiman1996Bagging,freund1995desicion}.
       Most, if not all, of these methods are either designed for pool-based
       sampling or for cases where observed data is stored. Bagging
       predictors~\cite{Breiman1996Bagging} proposes to improve performance of a
       single predictor by forming a committee from multiple versions of it,
       where the versions are trained on the bootstrap replicates of training
       data. This is followed by~\cite{Mamitsuka1998Boost} where diverse
       ensembles are generated using bagging and boosting techniques previously
       introduced by~\cite{freund1995desicion}. These strategies focus on a
       setting where the observed data is stored as opposed to our setting.
       Another popular ensemble learning algorithm, \textsc{Active-Decorate}
       relies on the existence of artificial training data to form a diverse set
       of examples. In our setting, however, we assume neither storing of
       previously seen data nor availability of artificial data. In the online
       setting, however, one could benefit from the strategy introduced
       in~\cite{freund1995desicion}. Upon seeing the label $c_t$, the authors
       propose to update the belief on the models such that $\dist_{t}\propto
       \dist_{t-1}\beta^{\loss_t}$. We note that, this update rule very closely resembles that of the structural query by committee, which we include in our numerical analysis. In fact, it is identical when both of $\beta$ are tuned to query budget $b$ amount of label in average over many realizations.

   As a disagreement measure, popular choices include vote margin, vote entropy and KL divergence between the label distributions of each committee member and the consensus in~\cite{Settles2008Sequence}. We first note that the latter two are equivalent for 0-1 loss functions $\ell$. The former, vote margin is measured by the difference between the votes of most voted and second most voted label. We omit this in our analysis motivated by the preliminary observation on the success of entropy over the vote margin.

     \item \underline{\smash{\emph{Importance Weighted Active Learning}}}:
 As indicated earlier, we implement the importance weighted active learning algorithm, introduced by~\cite{Beygelzimer2008ImpWeightedAL}. Formally, upon seeing a new instance $x_t$, the algorithm computes a rejection threshold $\theta_t$ 
     using sample complexity bounds, and update the hypothesis space $\Hc_t$ to contain only the models whose weighted error is $\theta_t$ greater than weighted error of the current best model at time $t$. The sampling probability $q_{t}$ is set to $\max_{i, j \in \mathcal{H}_t, c \in [C]} {\loss}^{(c)}_{t, i} - \loss^{(c)}_{t, j}$. We use 0-1 loss. Therefore, adaptation in our setting becomes making query decision based merely on the disagreement between the \emph{surviving} hypotheses at time $t$. That is, we query the label $c_t$ if and only if the surviving classifiers at time $t$ disagree on the labeling of $x_t$.

     \item \underline{\smash{\emph{Efficient Active Learning}}}: We adapt the efficient active learning algorithm presented by~\cite{Beygelzimer2010Agnostic, Beygelzimer2011EfficientAL}. In a manner similar to the importance weighted approach, the efficient active learning algorithm also uses the importance weighted framework. Upon receiving a new instance $x_t$, the algorithm measures the weighted error estimate between two competing models, and specifies a sampling probability based on a threshold that is a function of $\frac{C_0\log t}{t-1}$ for some parameter $C_0>0$. If the gap between the estimated weighted errors of two competing models are below this threshold, then the label $c_t$ is queried. Otherwise, the algorithm computes the sampling probability $q_t$ that is roughly {\small $\min\big \{1, \mathcal{O}({1}/{G_k^2}+{1}/{G_k})\frac{C_0\log k}{k-1}\big\}$} where {\small $G_k=\min_{i\in[k]}\Loss_{t, i}-\min_{j\in[k], j\neq i}\Loss_{t, j}$}. We refer to Algorithm~1 of~\cite{Beygelzimer2010Agnostic} for further details.
    %
    In our implementation, we consider the threshold parameter $C_0$ as hyperparameter and tune for efficient active learning algorithm to request amount of labels not exceeding the labelling budget $b$. However, as indicated in Figure~\ref{fig:comparison importance-weighted}, it underperforms the importance weighted active learning algorithm. However, it is crucial to emphasize again that these methods are meant to improve supervised training of classifiers instead of ranking of pretrained models. We include them in our comparison for the completeness.
\end{itemize}
\begin{figure*}[h!]
    \centering
      \includegraphics[width=1\linewidth]{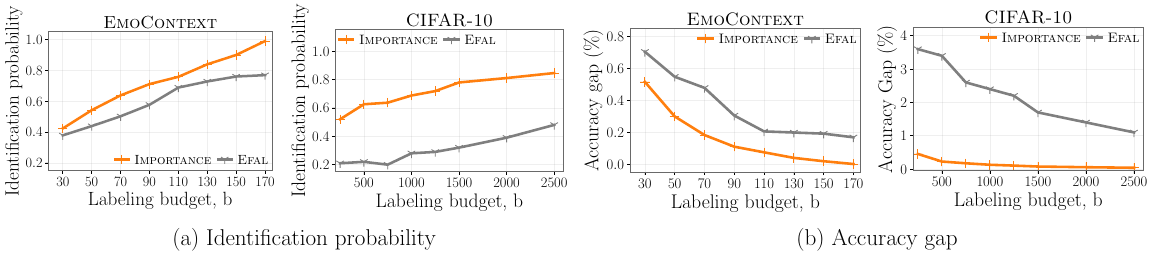}
        \caption{Comparison of importance weighted methods \{\importance, \efal\} and on the \emo\ and \cifar\ datasets}
        \label{fig:comparison importance-weighted}
        \vspace{-1em}
\end{figure*}


\subsection{Performance of \mpick\ on models with low accuracies}\label{app:experiments_detailed}

As mentioned in Section~\ref{sec:experiments}, we conduct another numerical analysis on the performance of $\mpick$ when pretrained models have relatively lower accuracies. Towards that, we train 80 models on \cifar\ varying in machine learning models and parameters. The accuracy of pretrained models line in 40-70$\%$ over a test set of of size 10 000. We compare the model selection methods over this new model collection by following the exact same procedure as in the Section~\ref{sec:experiments}. We use a stream size of 5\,000 and average the results over 500 realizations. Figure~\ref{fig:cifar10 V2 performance} summarize the comparison. When the accuracy of pre-trained models are low, the query by committee algorithm expectedly underperforms as the disagreement measure becomes noisy under the existence of models with low accuracies. $\mpick$, on the other hand, noticeably outperforms in returning the true best model as well as the ranking of the models (Figure~\ref{fig:cifar10 V2 performance}). The regret analysis in Figure~\ref{fig:cifar10 V2 performance} suggests that the structural query by committee method maintains a low regret throughout the streaming process as well as for different labeling budgets, and very closely followed by $\mpick$.
\begin{figure*}[h!]
    \centering
      \includegraphics[width=1\linewidth]{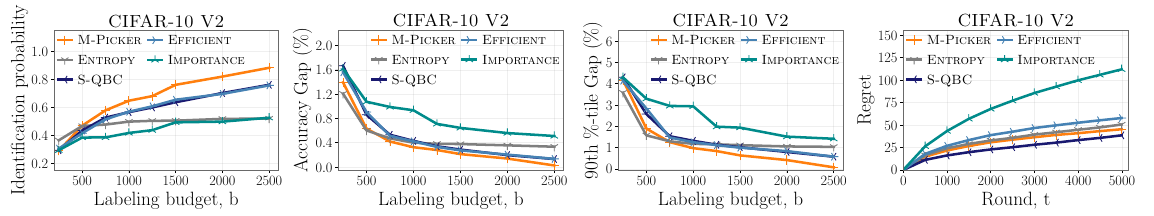}
        \caption{Performance evaluation of model selection methods on \cifarw\ dataset that consist of pre-trained models with low accuracies.}
        \label{fig:cifar10 V2 performance}
        \vspace{-1em}
\end{figure*}

\subsection{Hyperparameters}\label{app:hyperparameters}
The hyperparameter tuning is performed via grid search. For each grid point, we run the experiment for 100 realizations and compute the average number of requests. The grid search was performed over the following search space:
\begin{itemize}
    \item \underline{\smash{\emph{\cifar}}}: \mpick: \texttt{[0, 3\,000]}, \entropy: \texttt{[0, 20]}, \sqbc: \texttt{[0, 10]}, \importance: \texttt{[0, 0.9]}, \efal: \texttt{[0, 1.5e-2]}
    
    \item \underline{\smash{\emph{\imagenet}}}: \mpick: \texttt{[0, 135]}, \entropy: \texttt{[0, 22]}, \sqbc: \texttt{[0, 20]}, \importance: \texttt{[0, 1]}
    
    \item \underline{\smash{\emph{\drift}}}: \mpick: \texttt{[0, 60]}, \entropy: \texttt{[0, 4]}, \sqbc: \texttt{[0, 4]}, \importance: \texttt{[0, 05]}
     
    \item \underline{\smash{\emph{\emo}}}: \mpick: \texttt{[0, 60]}, \entropy: \texttt{[0, 4]}, \sqbc: \texttt{[0, 4]}, \importance: \texttt{[0, 05]}, \efal: \texttt{[0, 1e-2]}
        
    \item \underline{\smash{\emph{\cifarw}}}: \mpick: \texttt{[0, 1\,000]}, \entropy: \texttt{[0, 3]}, \sqbc: \texttt{[0, 10]}, \importance: \texttt{[0, 0.9]}, \efal: \texttt{[0, 1e-1]}

\end{itemize}
with grid size of 250 where grid points are equally spaced. The respective number of requests for each grid point can be found in our publicly available repository\footnote{\url{https://github.com/DS3Lab/online-active-model-selection}}.

Remark that the amount of requests by \mpick\ saturates when \mpick\ reaches at a high identification probability. Therefore, the update probability is upscaled with a very high value such that \mpick\ queries large number of labels, and thus comparison to other methods for large budget constraints are made possible. Practically, this would not be required as \mpick\ itself decides when to stop requesting labels. For example, when the update probability is upscaled by a factor of 11 for \cifarw\ dataset, the number of requests made by \mpick\ is 3\,800 labels, whereas an upscaling of 835 is used to enable \mpick\ requests nearly 4\,800 labels.

\end{document}